\documentclass{article}
\usepackage{ijcai19}

\usepackage{times}
\usepackage{soul}
\usepackage{url}
\usepackage[hidelinks]{hyperref}
\usepackage[utf8]{inputenc}
\usepackage[small]{caption}
\usepackage{graphicx}
\usepackage{amsmath}
\usepackage{booktabs}
\usepackage{amsthm}
\newtheorem{theorem}{Theorem}[section]

\usepackage{subcaption}
\usepackage{amssymb}

\DeclareMathOperator{\EX}{\mathbb{E}}

\usepackage{enumitem}
\usepackage{dblfloatfix}

\title{Rewarding High-Quality Data via Influence for Linear Regression}

\author{
Adam Richardson\and
Aris Filos-Ratsikas\And
Boi Faltings\\
\affiliations
\'Ecole polytechnique f\'ed\'erale de Lausanne, Switzerland \\
\emails
\{adam.richardson, aris.filosratsikas, boi.faltings\}@epfl.ch,
}

\begin{document}

\maketitle

\begin{abstract}
We consider a crowdsourcing data acquisition scenario, such as federated learning, where a Center collects data points from a set of rational Agents, with the aim of training a model. For linear regression models, we show how a payment structure can be designed to incentivize the agents to provide high-quality data as early as possible, based on a characterization of the influence that data points have on the loss function of the model.  Our contributions can be summarized as follows:  (a) we prove theoretically that this scheme ensures \emph{truthful} data reporting as a game-theoretic equilibrium and further demonstrate its robustness against mixtures of truthful and \emph{heuristic} data reports, (b) we design a procedure according to which the influence computation can be efficiently approximated and processed sequentially in batches over time, (c) we develop a theory that allows correcting the difference between the influence and the overall change in loss and (d) we evaluate our approach on real datasets, confirming our theoretical findings.
\end{abstract}

\section{Introduction}

The success of machine learning depends to a large extent on the availability of high quality data. For many applications, data has to be elicited from independent and sometimes self-interested data providers. A good example is federated learning \cite{konevcny2016federated}, where a single \emph{Center} (e.g. a large company) collects data from a set of \emph{Agents} to jointly learn a model.

So far, research in federated learning has focused on protecting the privacy of contributed data, but has not considered how to reward the Agents for the data they provided. A particular challenge is that these rewards should be scaled so that data is paid according to its quality, i.e., how useful it is for learning the model. Other examples of such settings can be found in \emph{crowdsourcing}. A prominent example of crowdsourcing is the Oxford English Dictionary\footnote{http://www.oed.com/}, which invites entries from the crowd on a voluntary basis, or more recently, the Open Street Map platform\footnote{https://www.openstreetmap.org/about}, an alternative to Google Maps, which is built on information provided by users.

Given that gathering accurate data is often a costly task, one shouldn't expect the Agents to exert the required effort unless they are properly compensated. This phenomenon has been documented in practice (e.g., see \cite{vuurens2011much,shah2015approval} and references therein) and is justified by the principles of agent rationality, an integral part of the field of \emph{game theory}. Game theory tells us that we should be rewarding agents monetarily, with their payments being dependent on the quality of the data they provide. However, to do this, we first need to answer the following question: ``What constitutes high-quality data?''

Intuitively, a high-quality data point is one that improves the accuracy of the model. We should be looking to give higher rewards to the agents that provide useful data, compared to those agents that do not contribute to improving the estimation. Additionally, we would like to reward more to agents that provide their data as early as possible.

\subsection{Our Approach}
We consider a crowdsourcing scenario like the one described above, and we aim to design incentive schemes, i.e., mechanisms that reward the Agents proportionally to the effect that they have on the accuracy of the model. A clear way of measuring the effect of individual points on the accuracy of a model is via the classical notion of \emph{influence} \cite{cook1980characterizations}. For a given data point, the influence quantifies how much the model's predictions would change if that point was not used in the training process. This allows us to quantify the effect that a single point has on the final outcome; we can simply remove the point, retrain, and compare the difference in the loss function. Based on the influence, we design schemes that reward agents proportionally to the decrease in the loss function due to their provided data points. This is desirable for the Center, since it will only have to pay for useful data and has the guarantee that the total cost of the data acquisition process will be bounded. 

We will show theoretically that the prescribed behavior, i.e., exerting the required effort to extract a sample from the underlying distribution and reporting that sample, is the best option for the agents under reasonable assumptions. In game-theoretic terms, we prove that our incentive schemes induce this type behavior as an equilibrium of the corresponding game. We strengthen our result by showing the robustness of the scheme against heuristic reports, i.e., agents that do not exert the effort to obtain useful information, and simply provide some untruthful report. Specifically, we show that (a) if the Center has independent data for testing, then the agents are always incentivized to provide truthful reports and (b) the same holds even if the test set is assembled by the reports of the agents, and a fraction of the agents decide to use an uninformed strategy.

For many practical applications computing the exact influence is prohibitively inefficient. For this reason, following \citeauthor{pmlr-v70-koh17a} [\citeyear{pmlr-v70-koh17a}], we compute an \emph{approximation of the exact influence}, based on up-weighing the training point by a small quantity. Our proposed approximation extends the idea in \cite{pmlr-v70-koh17a}, which turns out to be insufficient for our purposes. We show that the employed approximation is very close to the value of the exact influence, while achieving a notable improvement in speed for datasets of sufficiently high dimensions.

Then, we consider the case where the data points arrive sequentially; this captures most real-life scenarios of interest, as the data acquisition process is usually sequential. We employ our incentive scheme to reward the agents \emph{in batches}, with the fundamental property that agents that provide their data earlier are rewarded more, a desirable property as explained above. We consider two alternatives when it comes to the influence of data points in the current batch: \emph{M-Inclusive}, in which we include the data points of the current batch in training and \emph{M-Exclusive}, in which we do not. 

Another issue that arises when considering sequential rewarding is, ``How do we reward the first data points to arrive?'', as it is not clear how to measure the decrease in the loss function. For example, the Center might have prior knowledge that it can use to compute the initial influences.
In this paper, we adopt the simple approach of initializing the model with samples from a uniform distribution as a model of a knowledge-less initialization. The strength of this assumption is that it works in all scenarios, regardless of any possible availability of initial data points. We analytically show how the presence of the initial points affects the relation between the influence and the decrease in the loss function, for both the M-Inclusive and the M-Exclusive case.

Finally, we run experiments on several different real datasets, as well as generated data, and verify our theoretical results.

\subsection{Related Work}
The topic of learning a model when the input data points are provided by strategic sources has been the focus of a growing literature at the intersection of machine learning and game theory. A significant amount of work has been devoted to the setting in which Agents are interested in the outcome of the estimation process itself, e.g., when they are trying to sway the learned model closer to their own data points \cite{perote2004strategy,dekel2010incentive,meir2012algorithms,caragiannis2016truthful,chen2018strategyproof}. Our setting is concerned with the fundamental question of eliciting accurate data when data acquisition is costly for the agents, or when they are not willing to share their data without some form of monetary compensation. Another line of work considers settings in which the Agents have to be compensated for their loss of privacy \cite{cummings2015truthful,chen2018optimal}. 

A similar question to the one in our paper was considered by \citeauthor{Cai15} [\citeyear{Cai15}], where the authors design strategy-proof mechanisms for eliciting data and achieving a desired trade-off between the accuracy of the model and the payments issued. The guarantees provided, while desirable, are subject to certain strong assumptions. The authors assume that each agent chooses an \emph{effort level}, and the variance of the accuracy of their reports is a strictly decreasing convex function of that effort. Furthermore, these functions need to be known to the Center. Here, we only require that the cost of effort is bounded by a known quantity. Furthermore, our strategy space is more expressive in the sense that, as in real-life scenarios, data providers can choose which data to provide and not just which effort level to exert. A similar model to \cite{Cai15} was considered in \cite{westenbroek2017statistical,westenbroek2019competitive}, for the case of multiple Centers eliciting data from the crowd.

Our ideas are closely related to the literature of \emph{Peer Consistency} mechanisms \cite{faltings2017game} and \emph{Peer Prediction mechanisms for crowdsourcing} \cite{radanovic2016incentives}. The idea behind this literature is to extract high-quality information from individuals by comparing their reports against those of randomly chosen peers. This approach has been largely successful in eliciting \emph{truthful} information. The same principle applies to our case, where the payments are dependent on the improvement of the model and therefore agents are rewarded for providing \emph{helpful} information. Finally, \citeauthor{jia2019towards} [\citeyear{jia2019towards}] recently considered a setting in which the value of the provided data information is determined via the \emph{Shapley value}. Their approach is inherently different from ours, but it is worth noting that they consider the influence approximation of \cite{pmlr-v70-koh17a} for approximating the Shapley value.


\section{Setting and Foundations}\label{sec:setting}

In our setting, there is a \emph{Center} that wants to learn a model parametrized by $\theta$, with a non-negative loss function $L(z, \theta)$ on a sample $z=(x,y)$. In this paper, we will assume that the model is a linear regression model, but our general approach extends to other regression models as well. The samples are supplied by a set $\mathcal{A}$ of \emph{Agents}, with agent $i$ providing point $z_i=(x_i,y_i)$. We will denote by $\mathcal{A}_{-i}$ the set of agents without agent $i$. Given a set of data $Z = \{z_i\}_{i=1}^n$, the empirical risk is $R(Z, \theta) = \frac{1}{n}\sum_i L(z_i, \theta)$. 

It is assumed that each Agent $i$ must exert effort $e_i(z_i)$ to provide data point $z_i$. We adopt a simple effort model, in which the agent either exerts some effort $e$ to make an observation, or exerts zero effort, so $e_i(z_i)$ will be either $e$ or $0$. We assume that $e$ is fixed among agents, but we can easily adapt our results to the case of individual efforts $e_i$ per agent.


\subsubsection{Influence and Incentive Schemes} The Center must compensate the Agents for the effort they exert with some payment $p(z_i)$. We consider self-interested agents, that are trying to maximize the quantity $p_i(z_i) - e_i(z_i)$; we will refer to $z_i$ as the \emph{strategy} of agent $i$. For example, exerting effort and extracting a data point from the underlying distribution is a desirable strategy, whereas reporting random noise (e.g., sampling from the uniform distribution, with no regard to the underlying model) is an undesirable strategy. We refer to the former case as a \emph{truthful strategy} and to the latter case as a \emph{heuristic strategy}.

An \emph{incentive scheme} is a function that maps data points $z_i$ to payments $p(z_i)$; intuitively, a good incentive scheme should overcome the cost of effort (as otherwise agents are not incentivized to submit any observations) but also, crucially, to reward based on the effect that the data point $z_i$ has on improving the accuracy of the trained model. For this reason, we will design incentive schemes via the use of influences. Let $Z_{/j} = \{z_i\}_{i \neq j}$ and let
\begin{small}
$$
\hat{\theta} = \arg\min_{\theta} R(Z, \theta)\ \ \text{ and }\ \ \hat{\theta}_{/j} = \arg\min_{\theta} R(Z_{/j}, \theta).
$$ 
\end{small}
We will assume that the Center is in possession of an \emph{test set} $T = \{z_k\}$. Then the \emph{influence} of $z_j$ on the test set is defined as
\begin{small}
$$
\textrm{infl}(z_j,T,\theta)=R(T, \hat{\theta}_{/j}) - R(T, \hat{\theta}).
$$
\end{small}
We will simply write $\textrm{infl}(z_j)$, when $T$ and $\theta$ are clear from the context. Then, we can design incentive schemes based on the following general principle:
\begin{itemize}[leftmargin=*]
	\item[-] Compute the influence of the report $z_j$ of each Agent $j$ on the test set $T$.
	\item[-] Assign a score (payment) to the agent proportional to the influence of $z_j$.
	\item[-] Scale the scores appropriately to ensure that it is higher than the cost of effort $e_j(z_j)$ for agent $j$.
\end{itemize}
Following standard game-theoretic terminlogy, we will say that an agent supplying point $r_j$ is \emph{best responding} to the set of strategies $r_{-j}$ chosen by the other agents, if the strategy that it has chosen maximizes the quantity $\EX[p_i(r_j,r_{-j})-e_i(z_j)]$ over all possible alternative reports $z_j'$, where the expectation is over the distribution of reports of the other agents. We will say that a vector of strategies (i.e., a strategy profile) $(r_1, \ldots, r_n)$ is a \emph{Bayes-Nash equilibrium (BNE)} if for each agent $j$, $r_j$ is a best response. Below, we prove that the profile of truthful strategies is a BNE.

\begin{theorem}
	Suppose that the agents believe that their observations have mean $0$ error. Given an agent $j$, suppose agents $\mathcal{A}_{/j}$ all follow truthful strategies. Then, agent $j$ maximizes its expected influence by being following the truthful strategy as well.
\end{theorem}
\begin{proof}
	Let $o_j = z_j+\delta = (x_j, y_j+\delta)$, be the sample that agent $j$ observes, where $\delta$ is random variable with mean $0$. If agent $j$ reports $r_j$, its score is $R(T, \hat{\theta}_{/j}) - R(T, \hat{\theta})$. Since $R(T, \hat{\theta}_{/j})$ is fixed, then agent $j$ wishes to minimize $R(T, \hat{\theta})$. Let us assume that the true distribution of samples is $\Phi$, but that agent $j$ does not know $\Phi$. Then, in expectation, $R(T, \hat{\theta}) = \int_{\Phi} R(\phi, \hat{\theta}) d\phi$. Given a sufficient number of samples to compute $\hat{\theta}_{/j}$, $\int_{\Phi} R(\phi, \hat{\theta}) d\phi = R_{min}$, the minimum possible risk for the model family. Since agent $j$ has no knowledge of $\Phi$, then only by reporting $o_j$ can agent $j$ guarantee that in expectation he's moving $\hat{\theta}$ closer to the true optimum. Therefore agent $j$ maximizes his expected score by reporting $o_j$, i.e., by using the truthful strategy.
\end{proof}
The BNE result then follows immediately from the fact that, according to the incentive scheme, higher influences yield higher payments and the cost of effort is covered when the reports are truthful. 

\section{Methods}
\subsection{Influence Approximation}
Trying to practically implement an incentive-based payment mechanism imposes a host of challenges. The first is the computational cost of computing the influence for an agent. Specifically, we must compute $\hat{\theta}_{/j}$, which would involve entirely retraining the model. We present an approximation method based on the method described in \cite{pmlr-v70-koh17a}, which gives the following formula for the influence of $z_{j}$ on test point $z_{test}$:
\begin{small}
$$
\textrm{infl}(z_{test}, z_{j}) = \frac{1}{n}\nabla_\theta L(z_{test}, \hat{\theta}) H^{-1}_{\theta} \nabla_\theta L(z, \hat{\theta})
$$
\end{small}
This formula is derived by taking the first term of the Taylor expansion of the empirical risk with respect to $\theta$, which yields an approximation error of $O(1/n^2)$. However, this approximation has the undesirable property that the mean influence is $0$, by the definition of $\hat{\theta}$ as the solution to $\sum \nabla_\theta L(z, \hat{\theta}) = 0$. We can eliminate this property by including the second order term in the Taylor expansion of the empirical risk. Let $\partial \theta_j$ be the change in theta due to up-weighting a training point $z_j$, and let $H_i$ be the Hessian computed only on $z_i$.\\.
\begin{small}
$$
\partial \theta_j =  \frac{1}{n} H^{-1}_{\theta} \nabla_\theta L(z_i, \hat{\theta}) + \frac{1}{n^2} H^{-1}_{\theta} H_i H^{-1}_{\theta} \nabla_\theta L(z_i, \hat{\theta})
$$
\end{small}
We must also take into account the second order approximation of the change in the loss on a test point when computing the influence. Rather than approximate the change in test loss by $\textrm{infl}(z_\text{test}, z) = (\nabla_\theta L(z_{\text{test}}, \hat{\theta})) \cdot \partial \theta$, we take the second term in the Taylor expansion:
\begin{small}
$$
\textrm{infl}(z_\text{test}, z) = \left(\nabla_\theta L(z_{\text{test}}, \hat{\theta}) + \frac{1}{2} H_{\theta, z_\text{test}} \cdot \partial \theta \right) \cdot \partial \theta
$$
\end{small}
For the case of linear regression, computing $\hat{\theta}_{/j}$ for each data point can be computationally infeasible, but we show in Fig. \ref{fig:comptimes} that, while the computation time for the approximation grows faster with the number of test points, given a fixed test set the approximation will improve computation time given a high enough input dimension. For a model that is learned via an SGD method, it is clear that the influence approximation will provide significant improvements.

\begin{figure}[t]
\includegraphics[width=.45\textwidth]{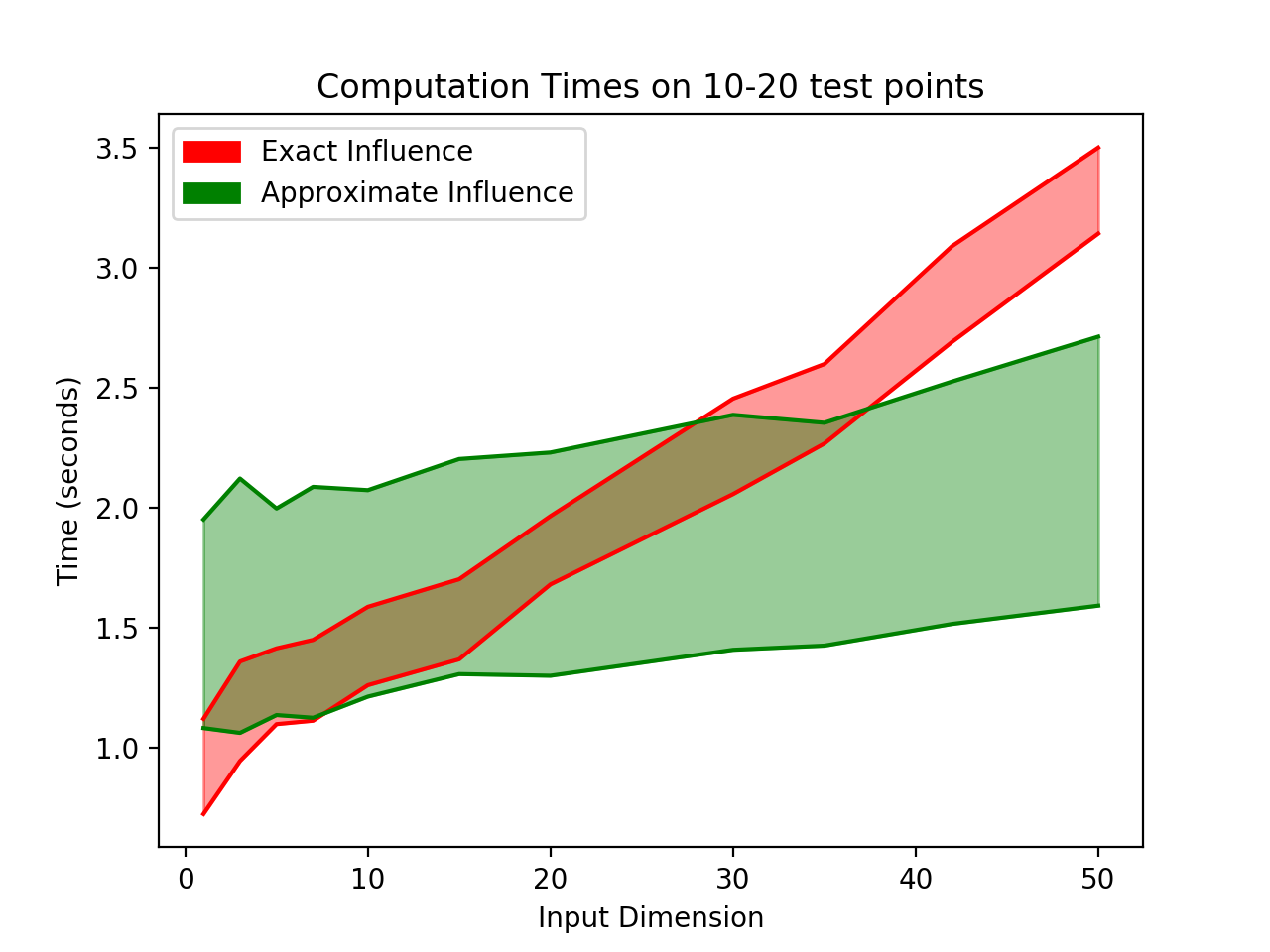}
\caption{Computation Times of Exact and Approximate Influences over 2000 training points.}
\label{fig:comptimes}
\end{figure}

\subsection{Batch Processing}
In a practical implementation, data arrives sequentially. Prior, we assumed that the influences would be computed over the entire dataset once all the data has been collected. 
Ideally, the Center could compute the influence and provide the payment immediately when a data point arrives. This has the additional advantage of allowing the Center to perform a priori budgeting. Suppose we have a dataset $\{z_i\}$ such that $i$ indicates the time of arrival of each datapoint. Then the sum of influences is the overall change in loss from the dataset.
\begin{small}
$$
\sum_{j=0}^n R(T, \hat{\theta}_{\{z_i\}_{i<j}}) - R(T, \hat{\theta}_{\{z_i\}_{i \leq j}}) = R(T, \hat{\theta}_{\emptyset}) - R(T, \hat{\theta}_{\{z_i\}})
$$
\end{small}
By assigning a utility to the overall change in risk, the Center can budget the entire data collection period before implementing the mechanism. However, computing the influence for each data point as it arrives can be computationally prohibitive, even using the influence approximation. The computation time of the approximation, in terms of complexity, is dominated by computing $H^{-1}$, which must be computed every time the model is updated. The Center can strike a balance between the two extremes by grouping the data into batches, such that $H^{-1}$ is only computed once per batch.\\
It is clear that with respect to a single batch, the game theory is the same as the one-batch case, however, we must now consider how batch processing affects incentives with respect to the time of reporting. We observe the following: the influence approximation presented in \cite{pmlr-v70-koh17a} has absolute error with respect to the exact influence of $O(\frac{1}{n^2})$, and the influence approximation is $0$ in expectation as shown in 3.1. Therefore, in expectation, the exact influence is $O(\frac{1}{n^2})$. With this it is clear that batch processing incentivizes Agents to report as early as possible, which is a desirable property for the Center. We show empirically that this is the case in Fig.~\ref{fig:inftime}.
\begin{figure}[t]
\includegraphics[width=.45\textwidth]{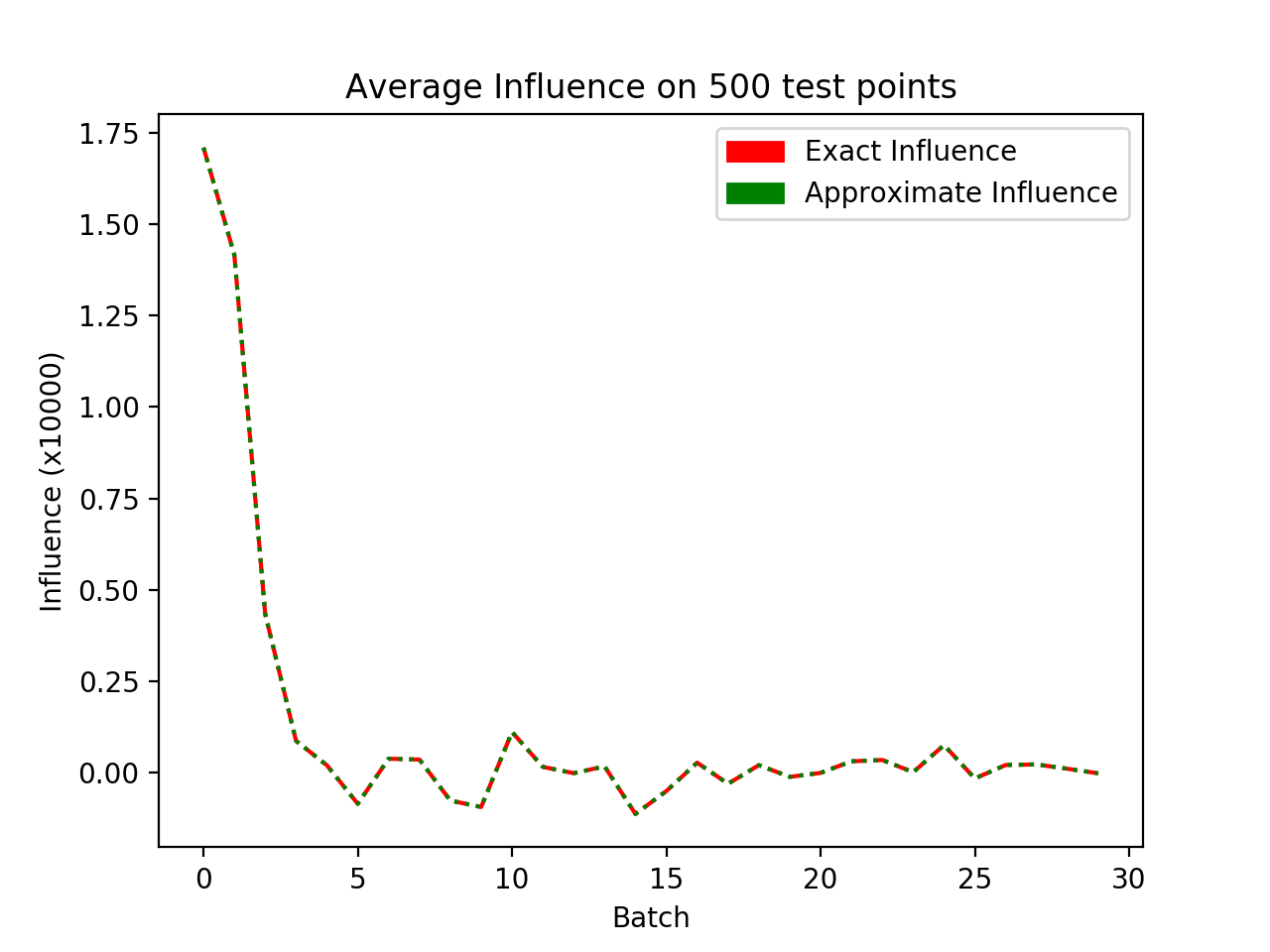}
\caption{Influence of 30 batches with batch size 100. Influence decreases rapidly after the first few batches.}
\label{fig:inftime}
\end{figure}

\subsection{Initializing the Model}
One point we did not address was the meaning of $\hat{\theta}_{\emptyset}$: the optimal parameters of the model given no training dataset. In many cases, for example for linear regression, this is not properly defined. We term this the \emph{initialization of the model}. We interpret this \emph{initial model} as the aggregate knowledge of the Center prior to the data collection period. If the Center already has some data it can use for training, then intuitively it won't need to collect as many data points or spend as much money during the data collection period. If the Center, rather than having a prior dataset, has some knowledge about the distribution of the model it wishes to learn, it can artificially generate a dataset by sampling from this prior distribution, with the number of samples corresponding to the confidence of the Center in the prior model. In the case where the Center has no prior knowledge, we consider this from an information theoretic perspective to be a state of maximum entropy, and therefore the initial model would be determined by sampling from the uniform distribution within the appropriate bounds.

\begin{table*}[!b]
	\centering
	\begin{tabular}{|c|c|c|c|c|c|c|}
		\hline
		\multicolumn{1}{|c|}{\textbf{Data Type}} & 
		\multicolumn{3}{c|}{\textbf{1st Order Approximation}} &
		\multicolumn{3}{c|}{\textbf{2nd Order Approximation}}
		\\ \hline
		\multicolumn{1}{|c|}{} & 
		\multicolumn{1}{c|}{\textbf{L1}} & 
		\multicolumn{1}{c|}{\textbf{Relative L1}} & 
		\multicolumn{1}{c|}{\textbf{L2}} & 
		\multicolumn{1}{c|}{\textbf{L1}} & 
		\multicolumn{1}{c|}{\textbf{Relative L1}} & 
		\multicolumn{1}{c|}{\textbf{L2}} 
		\\ \hline
		Linear Generated & 3.795e-06 & 4.806e-03 & 8.907e-11 & \textbf{2.470e-11} & \textbf{1.357e-08} & \textbf{7.218e-21}\\ \hline
		Red Wine & 2.414e-05 & 1.571e-02 & 8.269e-08 & \textbf{7.728e-08} & \textbf{1.091e-04} & \textbf{4.544e-12}\\ \hline
		White Wine & 3.925e-05 & 2.264e-02 & 7.929e-07 & \textbf{7.440e-07} & \textbf{1.347e-04} & \textbf{1.655e-09}\\ \hline
		Air Quality & 7.353e-05 & 1.117e-02 & 1.426e-06 & \textbf{3.244e-07} & \textbf{8.076e-06} & \textbf{1.386e-10}\\ \hline
		Crime & 1.745e-05 & 1.625e-01 & 6.963e-09 & \textbf{6.338e-07} & \textbf{2.653e-03} & \textbf{6.244e-11}\\ \hline
		Parkinsons & 1.696e-02 & 3.519e-02 & 3.054e-01 & \textbf{6.958e-03} & \textbf{2.478e-03} & \textbf{1.471e-01}\\ \hline
	\end{tabular}
	\caption{Error between the Exact Influence and the 1st and 2nd order approximations on 1000 agents evaluated on 200 test points.}
	\label{tab:approx_error}
\end{table*}

\subsection{M-Inclusive and M-Exclusive}
With batch processing, the Center has two further choices in how to implement the mechanism. The Center may include the most current batch in updating the model and compute the influence of each data point as though it were removed, or it could exclude the current batch and compute the influence of each data point as though it were added to the rest. We call these two methods \emph{M-Inclusive} and \emph{M-Exclusive} respectively. It is clear by construction that these two methods are equivalent with a batch size of one, in which case the sum of influences is equal to the overall change in risk. For the sake of computational efficiency, the Center may wish to choose a batch size $>1$. We restrict ourselves to the case of a linear regression model, but the analysis can be extended to any model in which the optimal parameters have a closed-form solution.\\
Let us consider two probability distributions $\Phi_1$ and $\Phi_2$, and we assume they describe an input-output relationship such that $\Phi(x,y) = q(x)p(y|x)$, and $q_1(x) = q_2(x)$. This assumption merely asserts that the data we are collecting is drawn from the same domain regardless of the distribution of the output. Distributions $\Phi_1$ and $\Phi_2$ determine, in expectation, models $M_1$ and $M_2$ respectively. Let us now define $R_{i,j}$ as the expected risk of model $M_i$ evaluated on distribution $\Phi_j$. Using the standard mean-squared-error loss function, we have that $R_{i,j} = R_{j,j} + \EX[(M_i - M_j)^2]$. Now suppose we sample $N_1$ points from $\Phi_1$ and $N_2$ points from $\Phi_2$ to form our training set $\{z\}$. Because the linear regression solution is linear with respect to $y$, and $q(x)$ is fixed, then $\{z\}$ determines in expectation a model $M_c = \frac{N_1 M_1 + N_2 M_2}{N_1 + N_2}$. With this, let us consider the practical application where $\Phi_1$ is the initialization distribution and $\Phi_2$ is the distribution of reports from the Agents. Then when we evaluate the model, we are only concerned with the error of the mixed model $M_c$ evaluated on $\Phi_2$:
\begin{small}
$$
R_{c,2} = R_{2,2} + \left(\frac{N_1}{N_1+N_2}\right)^2 \EX\left[(M_2 - M_1)^2\right]
$$
\end{small}
To simplify, we fix $N_1 = Q$ as the number of points used for initialization, we define $r = \EX[(M_2 - M_1)^2]$, and we let $N_2$ vary as $x$. Then we have our expected empirical risk in terms of $x$:
\begin{small}
$$
R(x) = \frac{Q^2 r}{(Q+x)^2} + R_{2,2}
$$
\end{small}
We can approximate the influence of a data point arriving after $x$ data points as the negative of the derivative of the risk:
\begin{small}
$$
-\frac{\partial R}{\partial x} = \frac{2Q^2 r}{(Q+x)^3}
$$
\end{small}
Then we can compute the overall change in loss with $n$ data points:
\begin{small}
$$
\Delta R = R(0) - R(n) = \frac{r n (2Q + n)}{(Q+n)^2}
$$
\end{small}
Now we consider the sum of influences of points in batches of batch size $b$. Consider the sum of influences for batch $k$:
\begin{small}
$$
S(k) = -bR'(kb) = \frac{2bQ^2 r}{(Q + kb)^3}
$$
\end{small}
Consider the sum of influences across all batches for both M-Inclusive and M-Exclusive:
\begin{small}
\begin{align*}
S_{\text{inc}} = 2bQ^2 r \sum_{k=1}^{n/b} \frac{1}{(Q+kb)^3}, 
S_{\text{exc}} = 2bQ^2 r \sum_{k=0}^{n/b-1} \frac{1}{(Q+kb)^3}
\end{align*}
\end{small} 
Comparing these to the change in risk, we get the following ratios in terms of the batch size:
\begin{small}
\begin{align*}
&D_{\text{inc}}(b) = \frac{S_{\text{inc}}}{\Delta R} = \frac{Q^2(Q+n)^2 [\psi''\left(\frac{Q+n}{b}+1\right) - \psi''(\frac{Q}{b}+1)]}{n(2Q+n)b^2}\\
&D_{\text{exc}}(b) = \frac{S_{\text{exc}}}{\Delta R} = \frac{Q^2(Q+n)^2 [\psi''(\frac{Q+n}{b}) - \psi''(\frac{Q}{b})]}{n(2Q+n)b^2}
\end{align*}
\end{small}
where $\psi''(x)$ is the second derivative of the Digamma function. By computing these values, the Center can pick an arbitrary batch size and normalize the scores such that the expected sum of influences is equal to the overall change in risk.

\begin{figure*}[t]
\begin{subfigure}{.5\textwidth}
\includegraphics[width=1\textwidth]{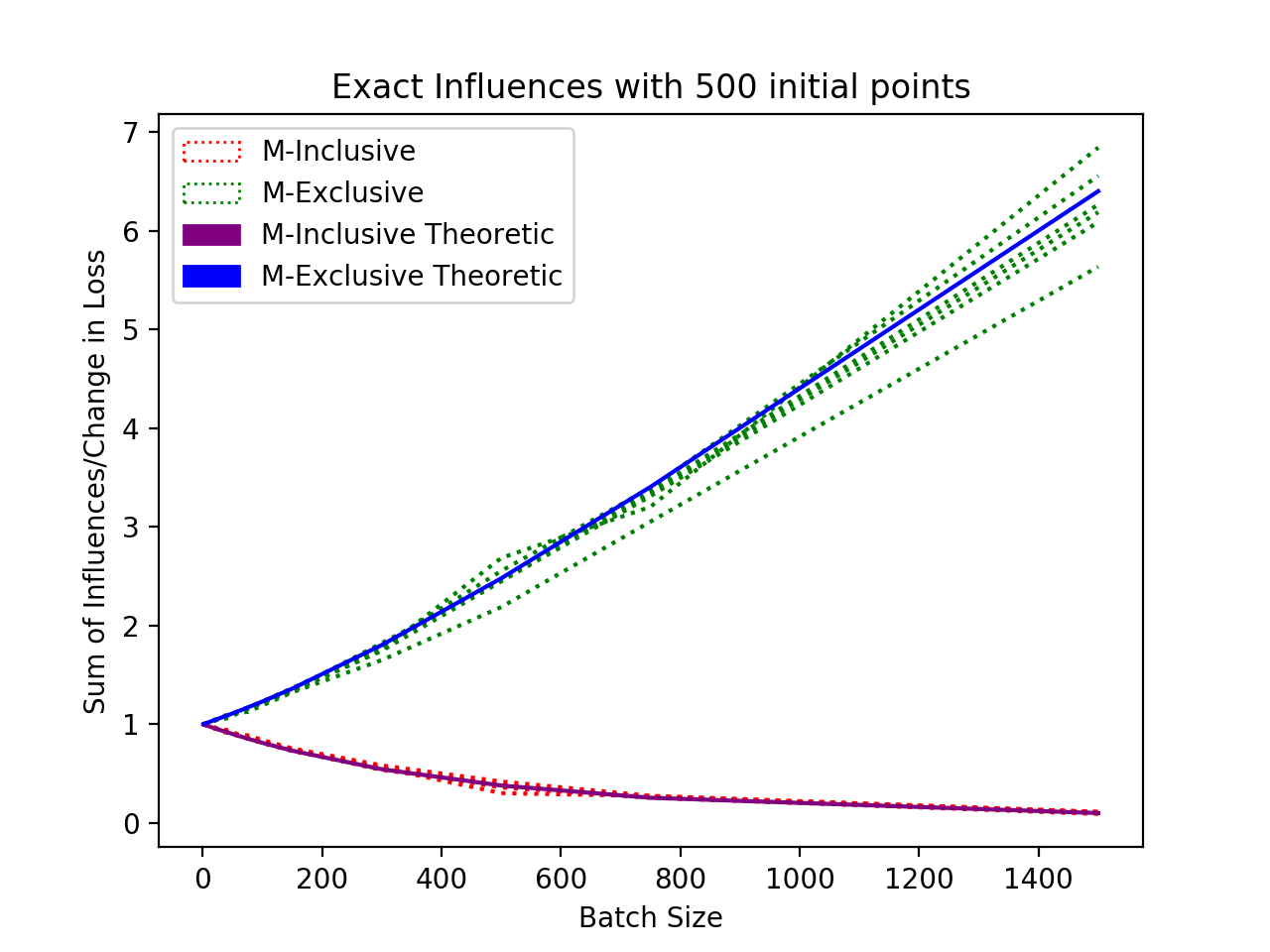}
\end{subfigure}
\begin{subfigure}{.5\textwidth}
\includegraphics[width=1\textwidth]{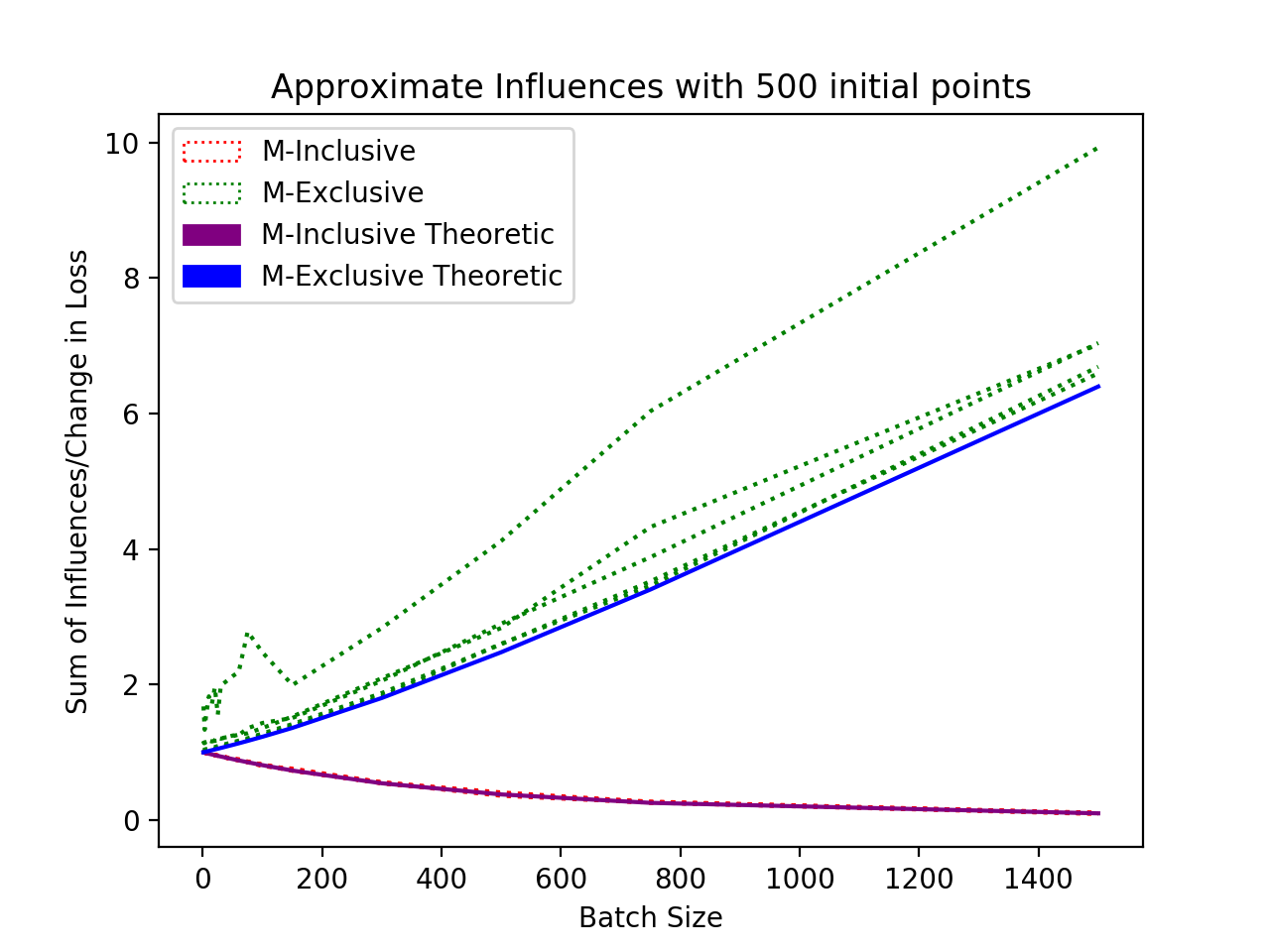}
\end{subfigure}
\caption{Ratio between Sum of Influences and Change in Loss with respect to batch size. }
\label{fig:MIncExc}
\end{figure*}

\subsection{Heuristic Reporting}
We have shown in Theorem 2.1 that under some reasonable assumptions, truthful reporting forms a Bayes-Nash Equilibrium. However, the mechanism must overcome the cost of effort, which has to be computed empirically. If the cost of effort is underestimated, then there may be agents who feel incentivized to report according to heuristics rather than to report truthfully. We can use the same analysis as in the previous subsection to show that these agents will be incentivized to opt-out, rather than play the heuristic strategy, in many realistic cases. Once again we consider two distributions $\Phi_1$ and $\Phi_2$ with the same assumptions as before. Let $\Phi_1$ be the distribution of heuristic reports and $\Phi_2$ by the distribution of truthful reports. Suppose we draw $x_1$ points from $\Phi_1$ and $x_2$ points from $\Phi_2$, with $n = x_1 + x_2$. Then we have our mixed model $M_c = \frac{x_1 M_1 + x_2 M_2}{n}$. Suppose the Center has an independent test set drawn entirely from $\Phi_2$. Then we have an expected empirical risk:
\begin{small}
$$
R_{c,2} = R_{2,2} + \frac{x_1^2 r}{n^2}
$$
\end{small}
We define $p = \frac{x_2}{n}$ as the probability that a data point is truthful, and we compute the influences of a data point drawn from $\Phi_1$ and of a data point drawn from $\Phi_2$:
\begin{small}
\begin{align*}
-\frac{\partial R_{c,2}}{\partial x_1} = - \frac{2rp(1-p)}{n},
-\frac{\partial R_{c,2}}{\partial x_2} = \frac{2r(1-p)^2}{n}
\end{align*}
\end{small}
We can see that for $p \in (0,1)$, $-\frac{\partial R_{c,2}}{\partial x_1}$ is always negative and  $-\frac{\partial R_{c,2}}{\partial x_2}$ is always positive. Thus, if the Center has an independent test set, Agents will always be incentivized to report according to the correct distribution, and agents that report according to heuristic will always receive a negative reward, and thus will not participate. 

In the general case, the Center won't have access to an independent test set. It will draw samples from the reports to build a test set. The mixed model $M_c$ will then be evaluated on the mixed distribution $\Phi_c = \frac{x_1 \Phi_1 + x_2 \Phi_2}{n}$. Taking the empirical risk $R_{c,c}$, we computing the following influences:
\begin{small}
\begin{align*}
&-\frac{\partial R_{c,c}}{\partial x_1} = \frac{p}{n}R_{1,1} - \frac{p}{n}R_{2,2} - \frac{p(2p-1)r}{n}\\
&-\frac{\partial R_{c,c}}{\partial x_2} = -\frac{1-p}{n}R_{1,1} + \frac{1-p}{n}R_{2,2} + \frac{(1-p)(2p-1)r}{n}
\end{align*}
\end{small}
We wish to know under what conditions the agents reporting according to $\Phi_2$ will receive a higher reward than those reporting according to $\Phi_1$:
\begin{small}
$$
0< -\frac{\partial R_{c,c}}{\partial x_2} - -\frac{\partial R_{c,c}}{\partial x_1}
$$
\end{small}
This yields:
\begin{small}
$$
p > \frac{1}{2} + \frac{R_{2,2} - R_{1,1}}{2r}
$$
\end{small}
Intuitively, this shows that if the inherent error of the correct model is greater than that of the heuristic model, then the mechanism requires more than half the reports to be truthful. Conversely, if the error of the correct model is less than that of the heuristic model, then the mechanism is more robust and does not require a majority of truthful reporters to maintain a proper incentive for truthful reporting.

\section{Experimental Results}

\paragraph{Datasets:} We present experimental results to demonstrate the validity of our methods in real scenarios. We start by enumerating the datasets used in our simulations:
\begin{itemize}[leftmargin=*]
	\item[-] \emph{Linear Generated}: We generate linear regression data as follows: pick an angle $\theta$ uniformly in $[-\pi/2, \pi/2]$, and a bias term from $N(0, 1)$. Using $\theta$ and the bias to determine a linear model, we uniformly sample $x \in [-1,1]$ and determine ground truth $y_{gt}$ values. We then add a noise variable drawn from $N(0, 1)$ to produce observations $y$.
	\item[-] \emph{Red Wine and White Wine}: UCI datasets with 11 attributes that predict a quality metric. \cite{cortez2009modeling}
	\item[-] \emph{Air Quality}: A UCI dataset with 15 attributes. We removed 6 attributes because they are either non-predictive or they have many missing values. We chose "C6H6(GT)" as the predicted attribute. \cite{de2008field}
	\item[-] \emph{Communities and Crime (Crime)}: A UCI dataset with 128 attributes. We remove 27 attributes because they are either not predictive or they have many missing values. We use "ViolentCrimesPerPop" as the predicted attribute. \cite{redmond2002data}
	\item[-] \emph{Parkinsons Telemonitoring (Parkinsons)}: A UCI dataset with 26 attributes. We removed 4 attributes because they are either non-predictive or they are redundant predicted attributes. We use "total\_UPDRS" as our predicted attribute. \cite{tsanas2009accurate}
\end{itemize}

\paragraph{Approximation Accuracy:}
We show in Table \ref{tab:approx_error} that on all datasets, our Second Order Approximation formula produces much better influence estimates in terms of absolute and relative absolute error than the First Order Approximation presented in \citeauthor{pmlr-v70-koh17a} [\citeyear{pmlr-v70-koh17a}].

\paragraph{M-Inclusive and M-Exclusive:}
We ran simulations to estimate the effect of batch size on the ratios $D_{\text{inc}}$ and $D_{\text{exc}}$. We ran each simulation with 1500 total training points with a varying batch size. Given a fixed batch size, we ran 10 trials for every dataset and aggregated them to form a more general estimate of $S_{\text{inc}}$, $S_{\text{exc}}$, and $\Delta R$. We then took the ratios of these aggregates and compared against our theoretical results for $D_{\text{inc}}$ and $D_{\text{exc}}$ in Fig. \ref{fig:MIncExc}. We ran this same simulation with different numbers of initial points 20, 100, 200, and 500. We have chosen only to show the case with 500 initial points, although the other simulations show the same relationship.

\section{Conclusion}
Federated learning with self-interested data providers requires incentives to ensure that data is of good quality, and that it is contributed as early as possible. For learning regression models, we have presented a novel incentive scheme based on influence functions that makes truthful reporting of accurate data a game-theoretic equilibrium. We have shown how influence can be approximated and processed in batches for efficiency, and developed a theory that allows correcting the difference between the influence and the overall change in loss. If the Center has a utility function that is non-linear with the loss of the model, this can even be passed on to the payments. We also analyzed the influence of agents that report heuristically without making the effort to collect actual data. We have empirically validated the theoretical results on multiple datasets. Although we have not addressed privacy in this paper, incentives are computed through matrix operations that can be carried out privately using multiparty computation \cite{du2004privacy}. We plan to investigate this and other options in future work.

\clearpage
\bibliographystyle{named}
\bibliography{references}

\begin{thebibliography}{}

\bibitem[\protect\citeauthoryear{Cai \bgroup \em et al.\egroup }{2015}]{Cai15}
Yang Cai, Constantinos Daskalakis, and Christos Papadimitriou.
\newblock Optimum statistical estimation with strategic data sources.
\newblock In Peter Grünwald, Elad Hazan, and Satyen Kale, editors, {\em
  Proceedings of The 28th Conference on Learning Theory}, volume~40 of {\em
  Proceedings of Machine Learning Research}, pages 280--296, Paris, France,
  03--06 Jul 2015. PMLR.

\bibitem[\protect\citeauthoryear{Caragiannis \bgroup \em et al.\egroup
  }{2016}]{caragiannis2016truthful}
Ioannis Caragiannis, Ariel Procaccia, and Nisarg Shah.
\newblock Truthful univariate estimators.
\newblock In {\em International Conference on Machine Learning}, pages
  127--135, 2016.

\bibitem[\protect\citeauthoryear{Chen \bgroup \em et al.\egroup
  }{2018a}]{chen2018optimal}
Yiling Chen, Nicole Immorlica, Brendan Lucier, Vasilis Syrgkanis, and Juba
  Ziani.
\newblock Optimal data acquisition for statistical estimation.
\newblock In {\em Proceedings of the 2018 ACM Conference on Economics and
  Computation}, pages 27--44. ACM, 2018.

\bibitem[\protect\citeauthoryear{Chen \bgroup \em et al.\egroup
  }{2018b}]{chen2018strategyproof}
Yiling Chen, Chara Podimata, Ariel~D Procaccia, and Nisarg Shah.
\newblock Strategyproof linear regression in high dimensions.
\newblock In {\em Proceedings of the 2018 ACM Conference on Economics and
  Computation}, pages 9--26. ACM, 2018.

\bibitem[\protect\citeauthoryear{Cook and
  Weisberg}{1980}]{cook1980characterizations}
R~Dennis Cook and Sanford Weisberg.
\newblock Characterizations of an empirical influence function for detecting
  influential cases in regression.
\newblock {\em Technometrics}, 22(4):495--508, 1980.

\bibitem[\protect\citeauthoryear{Cortez \bgroup \em et al.\egroup
  }{2009}]{cortez2009modeling}
Paulo Cortez, Ant{\'o}nio Cerdeira, Fernando Almeida, Telmo Matos, and Jos{\'e}
  Reis.
\newblock Modeling wine preferences by data mining from physicochemical
  properties.
\newblock {\em Decision Support Systems}, 47(4):547--553, 2009.

\bibitem[\protect\citeauthoryear{Cummings \bgroup \em et al.\egroup
  }{2015}]{cummings2015truthful}
Rachel Cummings, Stratis Ioannidis, and Katrina Ligett.
\newblock Truthful linear regression.
\newblock In {\em Conference on Learning Theory}, pages 448--483, 2015.

\bibitem[\protect\citeauthoryear{De~Vito \bgroup \em et al.\egroup
  }{2008}]{de2008field}
Saverio De~Vito, Ettore Massera, M~Piga, L~Martinotto, and G~Di~Francia.
\newblock On field calibration of an electronic nose for benzene estimation in
  an urban pollution monitoring scenario.
\newblock {\em Sensors and Actuators B: Chemical}, 129(2):750--757, 2008.

\bibitem[\protect\citeauthoryear{Dekel \bgroup \em et al.\egroup
  }{2010}]{dekel2010incentive}
Ofer Dekel, Felix Fischer, and Ariel~D Procaccia.
\newblock Incentive compatible regression learning.
\newblock {\em Journal of Computer and System Sciences}, 76(8):759--777, 2010.

\bibitem[\protect\citeauthoryear{Du \bgroup \em et al.\egroup
  }{2004}]{du2004privacy}
Wenliang Du, Yunghsiang~S Han, and Shigang Chen.
\newblock Privacy-preserving multivariate statistical analysis: Linear
  regression and classification.
\newblock In {\em Proceedings of the 2004 SIAM international conference on data
  mining}, pages 222--233. SIAM, 2004.

\bibitem[\protect\citeauthoryear{Faltings and
  Radanovic}{2017}]{faltings2017game}
Boi Faltings and Goran Radanovic.
\newblock Game theory for data science: eliciting truthful information.
\newblock {\em Synthesis Lectures on Artificial Intelligence and Machine
  Learning}, 11(2):1--151, 2017.

\bibitem[\protect\citeauthoryear{Jia \bgroup \em et al.\egroup
  }{2019}]{jia2019towards}
Ruoxi Jia, David Dao, Boxin Wang, Frances~Ann Hubis, Nick Hynes, Nezihe~Merve
  Gurel, Bo~Li, Ce~Zhang, Dawn Song, and Costas Spanos.
\newblock Towards efficient data valuation based on the shapley value.
\newblock In {\em Proceedings of the 22nd International Conference on
  Artificial Intelligence and Statistics (AISTATS)}, 2019.

\bibitem[\protect\citeauthoryear{Koh and Liang}{2017}]{pmlr-v70-koh17a}
Pang~Wei Koh and Percy Liang.
\newblock Understanding black-box predictions via influence functions.
\newblock In Doina Precup and Yee~Whye Teh, editors, {\em Proceedings of the
  34th International Conference on Machine Learning}, volume~70 of {\em
  Proceedings of Machine Learning Research}, pages 1885--1894, International
  Convention Centre, Sydney, Australia, 06--11 Aug 2017. PMLR.

\bibitem[\protect\citeauthoryear{Kone{\v{c}}n{\`y} \bgroup \em et al.\egroup
  }{2016}]{konevcny2016federated}
Jakub Kone{\v{c}}n{\`y}, H~Brendan McMahan, Felix~X Yu, Peter Richt{\'a}rik,
  Ananda~Theertha Suresh, and Dave Bacon.
\newblock Federated learning: Strategies for improving communication
  efficiency.
\newblock {\em arXiv preprint arXiv:1610.05492}, 2016.

\bibitem[\protect\citeauthoryear{Meir \bgroup \em et al.\egroup
  }{2012}]{meir2012algorithms}
Reshef Meir, Ariel~D Procaccia, and Jeffrey~S Rosenschein.
\newblock Algorithms for strategyproof classification.
\newblock {\em Artificial Intelligence}, 186:123--156, 2012.

\bibitem[\protect\citeauthoryear{Perote and
  Perote-Pena}{2004}]{perote2004strategy}
Javier Perote and Juan Perote-Pena.
\newblock Strategy-proof estimators for simple regression.
\newblock {\em Mathematical Social Sciences}, 47(2):153--176, 2004.

\bibitem[\protect\citeauthoryear{Radanovic \bgroup \em et al.\egroup
  }{2016}]{radanovic2016incentives}
Goran Radanovic, Boi Faltings, and Radu Jurca.
\newblock Incentives for effort in crowdsourcing using the peer truth serum.
\newblock {\em ACM Transactions on Intelligent Systems and Technology (TIST)},
  7(4):48, 2016.

\bibitem[\protect\citeauthoryear{Redmond and Baveja}{2002}]{redmond2002data}
Michael Redmond and Alok Baveja.
\newblock A data-driven software tool for enabling cooperative information
  sharing among police departments.
\newblock {\em European Journal of Operational Research}, 141(3):660--678,
  2002.

\bibitem[\protect\citeauthoryear{Shah \bgroup \em et al.\egroup
  }{2015}]{shah2015approval}
Nihar Shah, Dengyong Zhou, and Yuval Peres.
\newblock Approval voting and incentives in crowdsourcing.
\newblock In {\em International Conference on Machine Learning}, pages 10--19,
  2015.

\bibitem[\protect\citeauthoryear{Tsanas \bgroup \em et al.\egroup
  }{2009}]{tsanas2009accurate}
Athanasios Tsanas, Max~A Little, Patrick~E McSharry, and Lorraine~O Ramig.
\newblock Accurate telemonitoring of parkinson's disease progression by
  noninvasive speech tests.
\newblock {\em IEEE transactions on Biomedical Engineering}, 57(4):884--893,
  2009.

\bibitem[\protect\citeauthoryear{Vuurens \bgroup \em et al.\egroup
  }{2011}]{vuurens2011much}
Jeroen Vuurens, Arjen~P de~Vries, and Carsten Eickhoff.
\newblock How much spam can you take? an analysis of crowdsourcing results to
  increase accuracy.
\newblock In {\em Proc. ACM SIGIR Workshop on Crowdsourcing for Information
  Retrieval (CIR’11)}, pages 21--26, 2011.

\bibitem[\protect\citeauthoryear{Westenbroek \bgroup \em et al.\egroup
  }{2017}]{westenbroek2017statistical}
Tyler Westenbroek, Roy Dong, Lillian~J Ratliff, and S~Shankar Sastry.
\newblock Statistical estimation with strategic data sources in competitive
  settings.
\newblock In {\em 2017 IEEE 56th Annual Conference on Decision and Control
  (CDC)}, pages 4994--4999. IEEE, 2017.

\bibitem[\protect\citeauthoryear{Westenbroek \bgroup \em et al.\egroup
  }{2019}]{westenbroek2019competitive}
Tyler Westenbroek, Roy Dong, Lillian~J Ratliff, and S~Shankar Sastry.
\newblock Competitive statistical estimation with strategic data sources.
\newblock {\em arXiv preprint arXiv:1904.12768}, 2019.

\end{thebibliography}

\end{document}